\documentclass[twoside]{article}

\usepackage[utf8]{inputenc} %
\usepackage[T1]{fontenc}    %
\usepackage{hyperref}       %
\usepackage{url}            %
\usepackage{booktabs}       %
\usepackage{amsfonts}
\usepackage{amsmath} %
\usepackage{nicefrac,color}       %
\usepackage{microtype}      %
\usepackage{array,calc}
\usepackage{graphicx}
\usepackage{bm}
\usepackage{ulem}
\usepackage{cleveref}
\usepackage{natbib}
\usepackage{algorithm}
\usepackage{algorithmic}
\DeclareUnicodeCharacter{2009}{\,}
\usepackage[titlenumbered,ruled,noend,algo2e]{algorithm2e}

\SetCommentSty{mycommfont}
\SetEndCharOfAlgoLine{}

\usepackage{shortcuts_pa}
\usepackage{multirow}
\usepackage{comment}
\newtheorem{lemma}{Lemma}

\def\loss{\tilde{\cL}}
\def\peq{\mathrel{\mathop:}=}

\newcommand{\logdet}[1]{\log \lvert #1 \rvert}

\newcommand{\reb}[2]{{\iffalse #1 \fi } {#2}}
\usepackage[mathcal]{eucal}

\usepackage[accepted]{aistats2019}

\begin{document}

\twocolumn[

\aistatstitle{Stochastic algorithms with descent guarantees for ICA}

\aistatsauthor{Pierre Ablin \And Alexandre Gramfort \And Jean-Fran\c cois Cardoso \And Francis Bach}

\aistatsaddress{ INRIA \\ Université  Paris-Saclay \And INRIA \\ Université Paris-Saclay \And CNRS \\ Institut d'Astrophysique  \\ de Paris \And INRIA \\ \'Ecole Normale Supérieure }]

\begin{abstract}
Independent component analysis (ICA) is a widespread data exploration technique, where observed signals are modeled as linear mixtures of independent components.
From a machine learning point of view, it amounts to a matrix factorization problem with a statistical independence criterion.
Infomax is one of the most used ICA algorithms.
It is based on a loss function which is a non-convex log-likelihood.
We develop a new majorization-minimization framework adapted to this loss function.
We derive an online algorithm for the streaming setting, and an incremental algorithm for the finite sum setting, with the following benefits.
First, unlike most algorithms found in the literature, the proposed methods do not rely on any critical hyper-parameter like a step size, nor do they require a line-search technique.
Second, the algorithm for the finite sum setting, although stochastic, guarantees a decrease of the loss function at each iteration.
Experiments demonstrate progress on the state-of-the-art for large scale datasets, without the necessity for any manual parameter tuning.
\end{abstract}

\section{Introduction}
\label{sec:intro}
Independent component analysis (ICA)~\citep{comon1994independent} is an unsupervised data exploration technique.
In its classical and most popular form, it models a random vector $\bm{x} \in \bbR^{p \times 1}$ as a \emph{linear mixture} of independent sources.
This means that there exists a \emph{source} vector $\bm{s} \in \bbR^{p \times 1}$ of statistically independent features and a \emph{mixing matrix} $A \in \bbR^{p \times p}$, such that $\bm{x} = A \bm{s}$.
The aim of ICA is to recover $A$ from some realizations of $\bm{x}$
without any assumption or constraint on $A$.
Despite being a linear and shallow model,
ICA is widely used in many observational sciences.
Indeed, many physical phenomena are well modeled by ICA.
For example, in neuroscience, the physics driving the measurement process of electrical signals in the brain is linear following Maxwell's equations~\citep{Makeig30091997}.
In astronomy~\citep{morello2015revisiting}, mechanics~\citep{yang2014blind}, neuroscience~\citep{o2017concurrent}, biology~\citep{biton2014independent} and several other fields, ICA algorithms are used daily to process ever-increasing
amounts of data\footnote{Two of the most used ICA algorithms~\citep{bell1995information, hyvarinen1999fast} have been cited over 1500 times in 2017 according to Google Scholar}.
In some data processing pipelines, ICA can be a computational bottleneck for large datasets, calling for more scalable algorithms.
It is thus of importance to develop ICA solvers which are fast, easy to use and with strong convergence guarantees.

One of the first and most employed ICA algorithms is Infomax~\citep{bell1995information}.
The Infomax objective function is equivalent to a likelihood criterion in which each feature of $\bm{s}$ follows a \emph{super-Gaussian} distribution with density $d(\cdot)$
(roughly speaking, a super-Gaussian distribution is heavy-tailed;
a rigorous definition  is given in Section~\ref{sec:surrogate}).
The likelihood of $\bm{x}$ given~$A$ then is~\citep{pham1997blind}:
\begin{equation}
	p(\bm{x} \lvert A)
    =
    \frac{1}{\lvert \det(A) \rvert}
    \prod_{i = 1}^p d([A^{-1} \bm{x}]_{i}).
	\label{eq:likelihood}
\end{equation}
It is more convenient to work with the \emph{unmixing matrix} $W \peq A^{-1}$ and the negative log-likelihood, yielding a cost function $\ell(\bm{x}, W) \peq -\log(p(\bm{x} \lvert W^{-1}))$:
\begin{equation}
	\ell(\bm{x}, W) =
    - \logdet{W}
    - \sum_{i=1}^p \log(d([W\bm{x}]_{i})) \enspace.
\end{equation}
The underlying \emph{expected risk} is then:
\begin{align}\label{eq:true_loss}
	\cL(W)
    &\peq \bbE_{\bm{x}}[\ell(\bm{x}, W)] \\
	&= - \logdet{W}
    - \sum_{i=1}^p \bbE[\log(d([W \bm{x}]_{i}))]
    \enspace. \nonumber
\end{align}
Given a set of $n$ i.i.d. samples of $\bm{x}$, $X = [\bm{x}_1,\cdots, \bm{x}_n]\in \bbR^{p \times n}$, the \emph{empirical risk} reads:
\begin{align}\label{eq:emp_risk}
&\cL_n(W) \peq \frac1n \sum_{j=1}^n \ell(\bm{x}_j, W) \\
&= -\logdet{W}- \frac1n \sum_{i=1}^p\sum_{j=1}^n \log(d([W X]_{ij}))] \enspace. \nonumber
\end{align}
This article focuses on the inference of $W$ in two cases.
The first case is the \emph{finite-sum} setting:
using only $n$ samples, $W$ is found by minimizing $\cL_n$.
The second case is the \emph{online} setting, where a stream of samples arriving one by one is considered.
In this case, $n$ goes to infinity, and then $\cL_n$ tends towards $\cL$.
It is important to note that it is theoretically established~\citep{amari1997stability} and empirically observed that these criteria allow to unmix super-Gaussian sources even if their densities are different from $d$.

Although not formulated like this in the original article, ~\cite{cardoso1997Infomax} shown that Infomax solves the empirical risk minimization problem~\ref{eq:emp_risk}.
It does so by using a stochastic gradient method.
However, $\cL_n$ not being convex, it is hard to find a good step-size policy which fits any kind of data~\citep{bottou2016optimization}.
As a consequence, Infomax can take an extremely long time before it reaches convergence, or even fail to converge at all~\citep{montoya2017caveats}.
Still, the stochasticity of Infomax makes it efficient when the number of samples $n$ is large, because the cost of one iteration does not depend on~$n$.
On the other hand, several full-batch second-order algorithms have been derived for the exact minimization of $\cL_n$.
For instance, in~\citep{zibulevsky2003blind}, an approximation of the Hessian of $\cL_n$ is used to obtain a simple quasi-Newton method.
In~\citep{choi2007relative}, a trust region method is proposed using the same Hessian approximation.
More recently, ~\citet{ablin2017faster} proposed to use the L-BFGS algorithm with the Hessian approximations.
Full-batch methods are robust and sometimes show quadratic convergence speed, but an iteration can take a very long time when the number of samples $n$ is large.
They also crucially rely on a costly line-search strategy because of the non-convexity of the problem.

In this work, we make the following contributions:
\begin{itemize}
\item We introduce a set of surrogate functions for $\ell$, allowing for a majorization-minimization (MM) approach.
We show that this view is equivalent to an EM algorithm for ICA.
Consequently, techniques like incremental EM~\citep{neal1998view} and online EM~\citep{cappe2009line} can be efficiently applied to this problem.

\item Critically, the surrogate functions can be minimized in closed-form with respect to any single row of $W$.
Thus, the incremental algorithm guarantees the decrease of the surrogate loss at each iteration, without having to resort to expensive line-search techniques.
To the best of our knowledge, this feature is a novelty in the field of ICA algorithms.
\item Owing to a cheap partial update, the cost of one iteration of the proposed algorithm is similar to the cost of a stochastic gradient descent step.
Through experiments, the proposed methods are shown to perform better than the state-of-the-art, while enjoying the robust property of guaranteed decrease.
\end{itemize}

\paragraph{Notation.}
In the following, scalar values are noted in lower case (e.g. $y$), vectors in bold font (e.g. $\bm{x}$), and matrices in upper case (e.g. $W$).
For a square matrix $W$,  $\lvert W \rvert$ is the determinant of $W$.
For a matrix $M$, $M_{i:}$ denotes its $i$-th row, and $M_{:j}$ denotes its $j$-th column.
Given a function $u$ from $\bbR$ to $\bbR$ and a matrix $Y \in \bbR^{p \times n}$, $u(Y)$ denotes the matrix of element-wise operations: $\forall i, j, \enspace u(Y)_{ij} = u(Y_{ij})$.
For complexity analysis, we say that a quantity $Q$ is $O(\phi(n, p))$ if $\frac{Q}{\phi(n, p)}$ is bounded.

\section{Representations of super-Gaussian densities}

Super-Gaussian densities can be represented in at least two forms:
either variationally through a surrogate function, or probabilistically through a Gaussian scale mixture~\citep{palmer2006variational}.
These two representations lead to the same optimization algorithms but with a slightly different view point.

\subsection{Surrogate functions}
\label{sec:surrogate}

The density $d$ is assumed symmetric and \emph{super-Gaussian} in the sense that $-\log(d(\sqrt{x}))$ is an increasing concave function over $(0,  +\infty)$.
Following~\citep{palmer2006variational}, there exists a function $f$ such that:
\begin{equation}
	G(y) \peq -\log(d(y)) = \min_{u \geq 0} \frac{u y^2}{2} + f(u),
	\label{eq:conjugate}
\end{equation}
and the minimum is reached for a unique value denoted as $u^*(y)$.
Simple computations show that $u^*(y) = \frac{G'(y)}{y}$.
For $\bm{u} \in \bbR_+^{p \times 1}$, we introduce a new objective function $\tilde{\ell}(\bm{x}, W, \bm{u})$ that reads:
\begin{equation}
	\tilde{\ell}(\bm{x}, W, \bm{u})
	\peq
	- \logdet{W}
    +  \sum_{i=1}^p [\frac{1}{2}u_i\ [W\bm{x}]_{i}^2
    + f(u_i)],
\label{eq:newcost}
\end{equation}
and the associated empirical risk, for $U = [\bm{u}_1,\cdots,\bm{u}_n]\in \bbR_+^{p \times n}$:
\begin{align}\label{eq:newcost_emp}
	&\loss_n(W, U)
    \peq \frac1n \sum_{j=1}^n \tilde{\ell}(\bm{x}_j, W, \bm{u}_j) \\
    &=  - \logdet{W}
    + \frac{1}{n} \sum_{i=1}^p \sum_{j=1}^n [\frac12U_{ij}\ [WX]_{ij}^2
    +f(U_{ij})].
\end{align}
Following Eq.~\eqref{eq:conjugate}, we have:
\begin{lemma}[Majorization]
Let $W \in \bbR^{p \times p}$. For any $U \in \bbR_+^{p \times n}$, $\cL_n(W) \leq \loss_n(W, U)$, with equality if and only if $U = u^*(WX)$.
\label{lemma:majorization}
\end{lemma}
\begin{lemma}[Same minimizers]
	Let $W \in \bbR^{p \times p}$, and $U = u^*(WX)$. Then, $W$ 		minimizes $\cL_n$ if and only if $(W, U)$ minimizes $\loss_n$.
	\label{lemma:samemin}
\end{lemma}
\textbf{Proof:} Using the function $G$ introduced in Eq.~\eqref{eq:conjugate}, the loss $\mathcal{L}_n$ writes:
$$
\mathcal{L}_n(W) = - \logdet{W} + \frac1n\sum_{i=1}^p \sum_{j=1}^{n}G([W X]_{ij})
$$
For a given matrix $U \in\mathbb{R}^{p \times n}$, using Eq.~\eqref{eq:conjugate} we have for all $i, j$: $G([W X]_{ij}) \leq \frac12 U_{ij}[WX]_{ij}^2 + f([WX]_{ij}) $, with equality if and only if $U_{ij} = u^*([WX]_{ij})$.
Summing these equations yields as expected:
\begin{align*}
&- \logdet{W} + \frac1n\sum_{i=1}^p \sum_{j=1}^{n}G([W X]_{ij}) \leq
\\
&- \logdet{W} + \frac1n\sum_{i=1}^p \sum_{j=1}^{n}[\frac12 U_{ij}[WX]_{ij}^2 + f([WX]_{ij})]
\end{align*}
with equality if and only if for all $i, j$, $U_{ij} = u^*([WX]_{ij}). \enspace\square$

In line with the majorization-minimization (MM) framework~\citep{mairal2015incremental},
these two lemmas naturally suggest to minimize $\cL_n(W)$ by alternating the minimization of the auxiliary function $\loss_n(W, U)$  with respect to $W$ and $U$.
This will also be shown to be equivalent to the EM algorithm for the Gaussian scale mixture interpretation in the next Section.

The rest of the paper focuses on the minimization of $\loss_n$ rather than $\cL_n$, which yields the same unmixing matrix by Lemma~\ref{lemma:samemin}.

\subsection{EM algorithm with Gaussian scale mixtures}
\label{sec:em}

Super-Gaussian densities can also be represented as scale mixtures of Gaussian densities~\citep{palmer2006variational}, that is,
$d(y) = \int_0^{+\infty}
g(y, \eta)q(\eta) d\eta
$,
where $g(y, \eta) = \frac{1}{\sqrt{2\pi\eta}} \exp( -\frac{y^2}{2\eta})$ is a centered Gaussian density of variance $\eta$,
and $q(\eta)$ a distribution on the variance of the Gaussian distribution.
It turns out that the EM algorithm using the above form for our ICA model is exactly equivalent to the alternating optimization of $\loss_n$ (see a proof in the supplementary material).
The variable $U$ corresponds to the scale parameter in~\citep{palmer2006variational} and the EM algorithm alternates between setting $U$ to the posterior mean $u^*(Y)$ (E-step) and a descent move in $W$ (M-step).

\paragraph{Relationship to the noisy case.}
Many articles (e.g. ~\citep{palmer2006variational,girolami2001variational, bermond1999approximate}) have proposed EM-based techniques for the estimation of the latent parameters of the more general linear model:
\begin{equation}
	\bm{x} = A \bm{s} + \bm{n} \enspace,
	\label{eq:palmer_model}
\end{equation}
where $A$ is the mixing matrix, and $\bm{n} \sim \mathcal{N}(0, \Sigma)$ is a Gaussian variable accounting for noise.
In~\citep{palmer2006variational}, the matrix $A$ is assumed to be known, as well as the noise covariance $\Sigma$.
On the contrary, the present article deals with the case where $A$ is unknown, and where there is no noise.
The noisy case (with unknown $A$) is studied in e.g.~\citep{bermond1999approximate,girolami2001variational}.
An EM algorithm is derived for the estimation of $\bm{s}$, $A$ and $\Sigma$.
In the appendix, it is shown that this EM algorithm makes no progress in the limit of noise-free observations
since the EM update rule for $A$ becomes $A \leftarrow A$
when $\Sigma = 0$.
Hence, the EM algorithms found in the literature for the noisy case suffer considerable slowdown in high signal-to-noise regime.
In contrast, the approach derived in the following section is not affected by this problem.

\subsection{Examples}

Many choices for $G$ can be found in the ICA literature.
In the following, we omit irrelevant normalizing constants.
The original Infomax paper~\citep{bell1995information} implicitly uses $G(y) = \log(\cosh(y))$ since it corresponds to $G'(y) = \tanh(y)$ and $u^*(y)=\frac{\tanh(y)}{y}$.
This density model is one of the most widely used.
However, since an ICA algorithm has to evaluate those functions many times, using simpler functions offers significant speedups.
One possibility is to use a Student distribution: $G(y) = \frac12 \log(1 + y^2)$, for which $ u^*(y) = \frac{1}{1 + y^2}$.
In the following, we choose the Huber function: $G(y) = \frac12 y^2$ if $\lvert y \rvert < 1$ and $G(y) = \lvert y \rvert - \frac12$ if not.
This gives $u^*(y)= 1$ if $\lvert y \rvert < 1$, $u^*(y) =\frac{1}{\lvert y \rvert}$ otherwise.
\section{Stochastic minimization of the loss function}
\label{sec:algo}

Using a MM strategy,  $\loss_n(W, U)$ is minimized by alternating descent moves in $U$ and in $W$.
We propose an incremental technique which minimizes $\loss_n$ with a finite number of samples, and an online technique where each sample is only used once.
The pseudo code for these algorithms is given in Algorithms~\ref{alg:inc} and~\ref{alg:online}.
The difference between incremental and online technique only reflects through the variable $U$ which is estimated at the majorization step.
Hence, we first discuss the  minimization step.

\subsection{Minimization step: Descent in W}
\label{subsec:minw}

Expanding $[WX]_{ij}^2$, the middle term in the new loss function~\eqref{eq:newcost} is quadratic in the rows of $W$:
\begin{equation}
\loss_n = - \logdet{W} + \frac12\sum_{i=1}^p W_{i:} A^i W_{i:}^{\top} + \frac{1}{n} \sum_{i=1}^p \sum_{j=1}^n f(U_{ij}),
\label{eq:quad_form}
\end{equation}
where $W_{i:}$ denotes the $i$-th row of $W$, and the $A^i$'s are $p\times p$ matrices given by:
\begin{equation}
A^i_{kl} \peq \frac1n\sum_{j=1}^n U_{ij} X_{kj} X_{lj} \enspace .
\label{eq:acc}
\end{equation}
Therefore, when $U$ is fixed, with respect to $W$, $\loss_n$ is the sum of the $\log\det$ function and a quadratic term.
The minimization of such a function is difficult, mostly because the $\log\det$ part introduces non-convexity.
However, similarly to a coordinate descent move, it can be \emph{exactly partially} minimized in closed-form:
\begin{lemma}[Exact partial minimization]
Let $i \in [1, p]$, and $\bm{m} \in \bbR^{1 \times p}$ ($\bm{m}$ is a \emph{row} vector).
Consider the mapping $\Theta_i(\bm{m}) : \bbR^{1 \times p} \rightarrow \bbR^{p \times p}$  such that the matrix $\Theta_i(\bm{m})$ is equal to $I_p$, except for its $i$-th row  which is equal to $\bm{m}$.

Let $W\in \bbR^{p \times p}$ and $U \in \bbR^{p \times n}$.
Define $K \peq W A^i W^{\top}\in \bbR^{p \times p}$.
Then,
\begin{equation}
\argmin_{\bm{m} \in \bbR^{1\times p}}  \loss_n(\Theta_i(\bm{m})W, U) = \frac{1}{\sqrt{(K^{-1})_{ii}}}(K^{-1})_{i:} \enspace .
\label{eq:minw}
\end{equation}
\end{lemma}

\textbf{Proof:}
With respect to $\bm{m}$, $\loss_n(\Theta_i(\bm{m})W, U)$ is of the form $\phi(\bm{m})=-\log(\lvert m_i \rvert) + \bm{m}K\bm{m}^{\top}$.
Restraining to the region $m_i > 0$, this function is strongly convex and smooth, and thus possesses a single minimum found by cancelling the gradient.
Simple algebra shows :
$$
\nabla \phi(\bm{m}) = - \frac{1}{m_i}\bm{e}^i + \bm{m}K \enspace,
$$
where $\bm{e}^i$ is the $i$-th canonical basis vector.
Cancelling the gradient yields $ \bm{m} = \frac{1}{m_i} (K^{-1})_{i:}$, and inspection of the $i$-th coordinate of this relationship gives $m_i = \frac{(K^{-1})_{ii}}{m_i}$, providing the expected result. $\square$

In other words, we can exactly minimize the loss with a \emph{multiplicative update} of one of its rows.
Performing multiplicative updates on the iterate $W$ enforces the \emph{equivariance} of the proposed methods~\citep{cardoso1996equivariant}:
denoting by $\mathcal{A}$ the ``algorithm operator'' which maps input signals $X$ (be it a stream or a finite set) to the estimated mixing matrix, for any invertible matrix $B$, $\mathcal{A}(BX) =  B\mathcal{A}(X)$.

\subsection{Majorization step : Descent in U}

For a fixed unmixing matrix $W$, Lemma~\ref{lemma:majorization} gives: $\argmin_{U} \loss_n(W, U) = u^*(WX)$.
Such an operation works on the full batch of samples $X$.
When only one sample $X_{:j}=\bm{x}_j \in \bbR^{p \times 1}$ is available, the operation $U_{:j} \leftarrow u^*(W \bm{x}_j)$  minimizes $\loss_n(W, U)$ with respect to the $j$-th column of $U$.
As seen previously (Section~\ref{subsec:minw}), we only need to compute the $A^i$'s to perform a descent in $W$, hence one needs a way to accumulate those matrices.

\textbf{Incremental algorithm.}
To do so in an incremental way~\citep{neal1998view}, a memory $U^{\text{mem}} \in\mathbb{R}^{p\times n}$ stores the values of~$U$.
When a sample $\bm{x}_j$ %
is seen by the algorithm, we compute $U^{\text{new}}_{:j} = u^*(W \bm{x}_j)$, and update the $A^i$'s as:
\begin{equation}
A^i \leftarrow A^i + \frac1T (U^{\text{new}}_{ij} - U^{\text{mem}}_{ij}) \bm{x}_j\bm{x}_j^{\top} \enspace.
\label{eq:updateamem}
\end{equation}
The memory is then updated by $U^{\text{mem}}_{:j} \leftarrow U^{\text{new}}_{:j}$ enforcing $A^i = \frac1n\sum_{j=1}^n U^{\text{mem}}_{ij} \bm{x}_j\bm{x}_j^{\top}$ at each iteration.

\textbf{Online algorithm.}
When each sample is only seen once, there is no memory, and a natural update rule following~\citep{cappe2009line} is:
\begin{equation}
A^i \leftarrow (1 - \rho(n)) A^i + \rho(n) U_{ij} \bm{x}_j\bm{x}_j^{\top} \enspace,
\label{eq:updateastream}
\end{equation}
where $n$ is the number of samples seen, and $\rho(n)\in[0, 1]$ is a well chosen factor.
Setting $\rho(n) = \frac1n$ yields the unbiased formula $A^i(n) = \frac1n \sum_{j=1}^n U_{ij} \bm{x}_j\bm{x}_j^{\top}$.
A more aggressive policy $\rho(n) = \frac{1}{n^\alpha}$ for  $\alpha \in [\frac12, 1)$ empirically leads to faster estimation of the latent parameters. Note that since we are averaging sufficient statistics, there is no need to multiply $\rho(n)$ by a  constant.
\subsection{Complexity analysis}

\textbf{Memory:}
The proposed algorithm stores $p$ matrices $A^i$, which requires a memory of size $ \frac{p^2(p+1)}{2}$ (since they are symmetric).
In the incremental case, it stores the real numbers $U_{ij}$, requiring a memory of size $p \times n$.
In most practical cases of ICA, the number of sources $p$ is very small compared to $n$ ($n \gg p$), meaning that the dominating memory cost is $p \times n$.
In the online case, the algorithm only loads one mini-batch of data at a time, leading to a reduced memory size of $p \times n_b$, where $n_b$ is the mini-batch size.

\textbf{Time:}
The majorization step requires to update each coefficient of the matrices $A^i$'s, meaning that it has a time complexity of $p^3 \times n_b$.
The minimization step requires to solve $p$ linear systems to obtain the matrices $K_{i:}^{-1}$.
Each one takes $O(p^3)$.
An improvement based on preconditioned conjugate gradient method~\citep{shewchuk1994introduction} is proposed in the appendix to reduce the computational cost.
The total cost of the minimization step is thus $O(p^4)$.
In practice, $p \ll n_b$, so the overall cost of one iteration is dominated by the majorization step, and is $ \frac{p^2(p+1)}{2} \times n_b$.
A stochastic gradient descent algorithm with the same mini-batch size $n_b$, as described later in Section~\ref{sec:expe}, has a lower time complexity of $p^2 \times n_b$.
We now propose a way to reach the same time complexity with the MM approach.

\subsection{Gap-based greedy update}
\label{sec:dual}
In order to reduce the complexity by one order of magnitude in the majorization step, only a subset of fixed size $ q < p$ of the matrices $A^i$ is updated for each sample.
Following Eq.~\eqref{eq:conjugate},  it is given by what we call \emph{gap} : a positive quantity measuring the decrease in $\loss_n$ provided by updating $U_{ij}$.
In the following, define $\tilde{U}_{i'j'} \peq U^{\text{mem}}_{i'j'}$ if $(i', j') \neq (i, j)$, and $\tilde{U}_{ij} \peq U^{\text{new}}_{ij} = u^*([WX]_{ij})$. The gap is given by:
\begin{align}\label{eq:dualgap}
\text{gap}(W, U^{\text{mem}}_{ij})
\peq \loss_n(W, U^{\text{mem}}) - \loss_n(W, \tilde{U}) \\
=
\frac{1}{2}U^{\text{mem}}_{ij}\ [WX]_{ij}^2
+ f(U^{\text{mem}}_{ij}) - G([WX]_{ij})\enspace .
\end{align}

Since all the above quantities are computed during one iteration anyway, computing the gap for each signal $i \in [1, p]$ only adds a negligible computational overhead, which scales linearly with~$p$.
Then, in a greedy fashion, only the coefficients $U_{ij}$ corresponding to the $q$ largest gaps are updated, yielding the largest decrease in $\loss_n$ possible with $q$ updates.
In the experiments (Figure~\ref{fig:greed}), we observe that it is much faster than a random selection, and that it does not impair convergence too much compared to the full-selection ($q=p$).
In the online setting, there is no memory, so we simply choose $q$ indices among $p$ at random.

\textbf{Related work:}
The matrices $A^i$ are \emph{sufficient statistics} of the surrogate ICA model for a given value of $U$.
The idea to perform a coordinate descent move~\eqref{eq:minw} after each update of the sufficient statistics is inspired by online dictionary learning~\citep{mairal2009online}\reb{}{, Gaussian graphical models~\citep{honorio2012variable}} and non-negative matrix factorization~\citep{lefevre2011online}.

{\fontsize{4}{4}\selectfont
\begin{algorithm}[t]
\SetKwInOut{Input}{Input}
\SetKwInOut{Init}{Init}
\SetKwInOut{Parameter}{Param}
\caption{Incremental MM algorithm for ICA}
\Input{Samples $X \in \bbR^{p \times n}$}
\Parameter{Number of iterations $t_{\text{max}}$, mini-batch size $n_b$, number of coordinates to update per sample $q$}
\Init{Initialize $W = I_p$, $U^{\text{mem}}= 0 \in \bbR^{n \times p}$ and $A^i= 0 \in \bbR^{p \times p},
    \enspace \forall i \in [1, p]$}
\For{$t = 1, \dots, t_{\text{max}}$}
  {
    Select a mini-batch $b$ of size $n_b$ at random\\
    \For(\tcp*[f]{Majorization}){each index $j\in b$}{
    	Select $\bm{x} = X_{:j}$\\
    	Compute $\bm{u}^\text{new} = u^*(W\bm{x})$\\
    	Compute the gaps~\eqref{eq:dualgap}\\
    	Find the $q$ sources $i_1, \dots, i_q$ corresponding to the largest gaps\\
      Update $A^i$ for $i=i_1,\dots, i_q$ using Eq.~\eqref{eq:updateamem}\\
      Update the memory: $U^{\text{mem}}_{:j} = \bm{u}^{\text{new}}$
    }
    \For(\tcp*[f]{Minimization}){$i=1,\dots,p$}{
    	Update the $i$-th row of $W$ using Eq.~\eqref{eq:minw}
    }
  }
\Return{W}
\label{alg:inc}
\end{algorithm}
}

{\fontsize{4}{4}\selectfont
\begin{algorithm}[t]
\DontPrintSemicolon
\SetKwInOut{Input}{Input}
\SetKwInOut{Init}{Init}
\SetKwInOut{Parameter}{Param}
\caption{Online MM algorithm for ICA}
\Input{A stream of samples $X$ in dimension $\bbR^p$}
\Parameter{Number of iterations $t_{\text{max}}$, mini-batch size $n_b$, number of coordinates to update per sample $q$}
\Init{Initialize $W = I_p$ and $A^i= 0 \in \bbR^{p \times p}, \enspace \forall i \in [1, p]$}

  \For{$t = 1, \dots, t_{\text{max}}$}{
    Fetch $n_b$ samples from the stream \\
    \For(\tcp*[f]{Majorization}){each fetched sample $\bm{x}$}{
    	Compute $\bm{u} = u^*(W\bm{x})$ \\
    	Compute $q$ indices  $i_1, \dots, i_q$  at random \\
      Update $A^i$ for $i=i_1,\dots,i_q$ using Eq.~\eqref{eq:updateastream}
    }
    \For(\tcp*[f]{Minimization}){$i=1\dots,p$}{
    	Update the $i$-th row of $W$ using Eq.~\eqref{eq:minw}
    }
  }
\Return{W}
\label{alg:online}
\end{algorithm}
}

\section{Experiments}
\label{sec:expe}

In this section, we compare the proposed approach to other classical methods to minimize $\cL$.
The code for the proposed methods is available online at \url{https://github.com/pierreablin/mmica}.
\subsection{Compared algorithms}
\textbf{Stochastic gradient descent (SGD).}
Given a mini-batch $b$ containing $n_b$ samples, the relative gradient $\nabla(\cL_n)_{ik} = \frac{1}{n_b} \sum_{j\in b} G'([WX]_{ij})[WX]_{kj}$ is computed.
Then, a descent move $W \leftarrow (I_n - \rho \nabla(\cL_n)) W$ is performed.
The choice of the step size $\rho$ is critical and difficult.
The original article uses a constant step size, but more sophisticated heuristics can be derived.
This method can be used both for the finite sum and the online problem.
It is important to note that once $WX$ and $G'(WX)$ are computed, it needs twice as many elementary operations to compute the gradient as it takes to update one matrix $A^i$ (Eq.~\eqref{eq:updateamem} and Eq.\eqref{eq:updateastream}) when $n_b \gg p$.
The first computation requires $n_b \times p^2$ operations, while the second takes $ n_b \times \frac{p(p+1)}{2}$ (since the matrices $A^i$ are symmetric).
When $n_b$ is large enough, as it is the case in practice in the experiments, these computations are the bottlenecks of their respective methods.
Hence, we take $q=2$ in the experiments for the MM algorithms, so that the theoretical cost of one iteration of the proposed method matches that of SGD.

\textbf{Variance reduced methods.}
One of the drawbacks of the stochastic gradient method is its sub-linear rate of convergence, which happens because the stochastic gradient is a very noisy estimate of the true gradient.
Variance reduced methods such as SAG~\citep{schmidt2017minimizing}, SAGA~\citep{defazio2014saga} or SVRG~\citep{johnson2013accelerating} reduce the variance of the estimated gradient, leading to better rates of convergence.
However, these methods do not solve the other problem of SGD for ICA, which is the difficulty of finding a good step-size policy.
We compare our approach to SAG, which keeps the past stochastic gradients in memory and performs a descent step in the averaged direction.
This approach is however only relevant in the finite-sum setting.

\textbf{Full batch second order algorithms.}
We compare our approach to the ``Fast-Relative Newton'' method (FR-Newton)~\citep{zibulevsky2003blind} \reb{}{and the ``Preconditioned ICA for Real Data'' algorithm (Picard)~\citep{ablin2017faster}.}
The former performs quasi-Newton steps using a simple approximation of the Hessian of $\cL_n$, which is as costly to compute as a gradient.
\reb{}{The later refines the approximation by using it as a preconditioner for the L-BFGS algorithm. For both algorithms, }
one iteration requires to compute the gradient and the Hessian on the full dataset, resulting in a cost of $2 \times p^2 \times n$, and to evaluate the gradient and loss function for each point tested during the line search, so the overall cost is $(2 + n_{\text{ls}}) \times p^2 \times n$ where $n_{\text{ls}} \geq 1$ is the number of points tested during the line-search.
Thus, one epoch requires more than $3$ times more computations than one of SGD or of the proposed algorithms.
These algorithms cannot be used online.

\textbf{Full batch MM.}
For the finite-sum problem, we also compare our approach to the full-batch MM, where the whole $U$ is updated at the majorization step.

\reb{}{
\textbf{FastICA.}
FastICA~\citep{hyvarinen1999fast} is a full batch fixed point algorithm for ICA.
It \emph{does not solve the same optimization problem} as the one presented in this paper (it does not minimize $\mathcal{L}_n$, see~\citep{hyvarinen1999fixed}).
Hence, we do not include metrics involving $\mathcal{L}_n$ to benchmark it.
However, it is one of the most widely used algorithms for ICA in practical applications, and is popular for its fast estimation speed.
Furthermore, it is shown to have similar convergence properties as FR-Newton~\citep{ablin2018faster}
}
\subsection{Performance measures}

The following quality measures are used to assess the performance of the different algorithms:

\textbf{Loss on left-out data:} It is the value of the loss on some data coming from the same dataset but that have not been used to train the algorithms.
This measure, which boils down to the likelihood of left-out data, is similar to the \emph{testing error} in machine learning, and can be computed in both the streaming and finite-sum settings.

\textbf{Amari distance~\citep{moreau1998self}:} When the true mixing matrix $A$ is available,  for a matrix $W$, the product $R=WA$ is computed, and the \emph{Amari distance} is given by:
$$
\sum_{i=1}^p \left(\sum_{j=1}^p \frac{R_{ij}^2}{\max_l R_{il}^2} -1\right) + \sum_{i=1}^p \left(\sum_{j=1}^p \frac{R_{ji}^2}{\max_l R_{lj}^2} -1\right) .
$$
This distance measures the proximity of $W$ and $A^{-1}$ up to scale and permutation indetermination.
It cancels if and only if $R$ is a scale and permutation matrix, i.e., if the separation is perfect.
This measure is relevant both for the online and finite-sum problems.
\reb{}{It is the only metric for which it makes sense to compare FastICA to the other algorithms since it does not involve the loss function.}

\textbf{Relative gradient norm:} The norm of the full-batch relative gradient  of $\cL_n$ is another measure of convergence.
Since the problem is non-convex, the algorithms may converge to different local minima, which is why we favor this metric over the \emph{train error}.
It is however only relevant in the finite-sum setting.
In this setting, a converging algorithm should drive the norm of the full-batch relative gradient to zero.

\subsection{Parameters and initialization}
The stochastic algorithms (SGD, SAG, and the proposed MM techniques) are used with a batch size of $n_b = 1000$.
The proposed MM algorithms are run with a parameter $q=2$, which ensures that each of their iterations is equivalent to one iteration of the SGD algorithm.
In the online setting, we use a power $\alpha = 0.5$ to speed up the estimation.
The step-sizes of SGD and SAG are chosen by trial and error on each dataset, by finding a compromise between convergence speed and accuracy of the final mixing matrix.
In the online case, the learning rate is chosen as $\lambda \times n^{-0.5}$ for SGD.
FR-Newton and Picard are run with its default parameters.

Regarding initialization, it is common to initialize an ICA algorithm with an approximate \emph{whitening} matrix.
A whitening matrix $W$ is such that the signals $W\bm{x}$ are decorrelated.
It is interesting to start from such a point in ICA because decorrelation is a necessary condition for independence.
Denoting $C_{\bm{x}}$ the correlation matrix of the signals, the whitening condition writes $WC_{\bm{x}}W^{\top} = I_p$.
Hence, the whitening matrices are the $W=RC_{\bm{x}}^{-\frac12}$ where $R$ is a rotation ($R^{\top}R = I_p$).
In practice, we take $R=I_p$.
The covariance matrix needs to be estimated.
In the case of a fixed dataset $X \in \bbR^{p\times n}$, we can use the empirical covariance $\tilde{C}_X = \frac1nXX^{\top}$ as an approximation.
However, the cost of such a computation, $O(p^2 \times n)$, gets prohibitively large as $n$ grows.
Since the whitening is only an initialization, it needs not be perfectly accurate.
Hence, in practice, we compute the empirical covariance on a sub-sampled version of $X$ of size $n=10^4$.
The same goes for the online algorithm: we fetch the first $10^4$ samples to compute the initial approximate whitening matrix.

\subsection{Datasets}
\label{sec:sim}
\textbf{Synthetic datasets:}
For this experiment, we generate a matrix $S \in \bbR^{p\times n}$ with $p = 10$ and $n=10^6$ of independent sources following a super-Gaussian Laplace distribution: $d(x) = \frac12 \exp(-\lvert x\rvert)$.
Note that this distribution does not match the Huber function used in the algorithms, but estimation is still possible since the sources are super-Gaussian.
Then, we generate a random mixing matrix $A \in \bbR^{p\times p}$ of normally distributed coefficients.
The algorithms discussed above are then run on $X = AS$, and the sequence of iterates produced is recorded.
Finally, the different quality measures are computed on those iterates.
We repeat this process $100$ times with different random realizations, in order to increase the robustness of the conclusions.
The averaged quality measures are displayed in Fig.~\ref{fig:res_synth}.
In order to compare different random initializations, the loss evaluated on left-out data is always shifted so that its plateau is at $0$.

To observe the effect of the greedy gap selection, we generate another dataset in the same way with $p=30$, $n=10^5$. Results are displayed in Fig.~\ref{fig:greed}.
\begin{figure}
  \begin{minipage}{\columnwidth}
    \centering
    \includegraphics[width=\columnwidth]{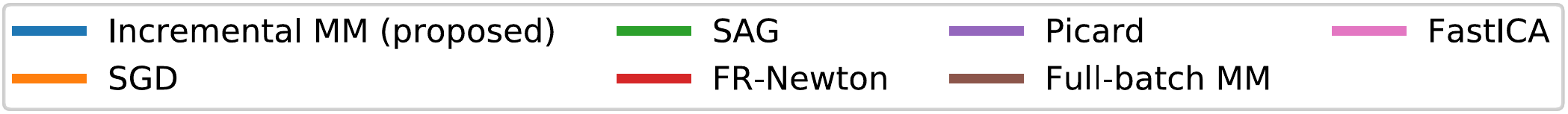}
  \end{minipage}
  \begin{minipage}{\columnwidth}
    \includegraphics[width=0.32\columnwidth]{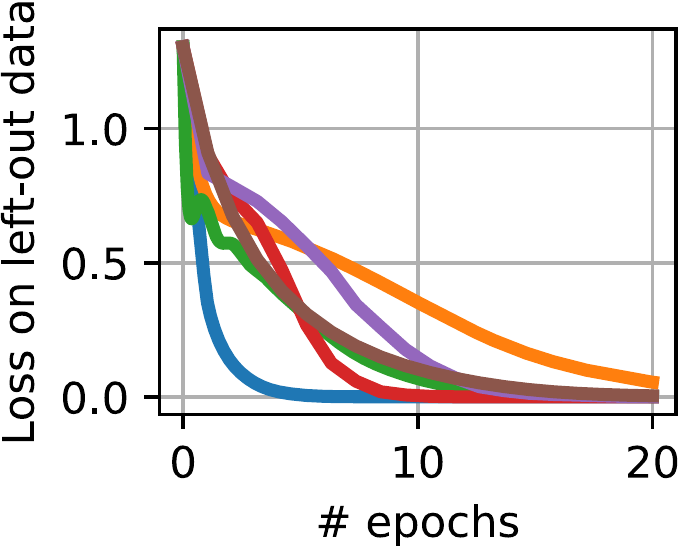}
    \includegraphics[width=0.32\columnwidth]{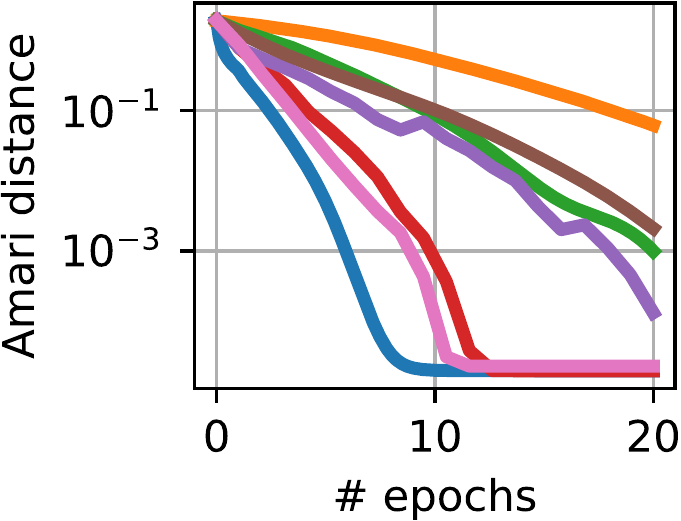}
    \includegraphics[width=0.32\columnwidth]{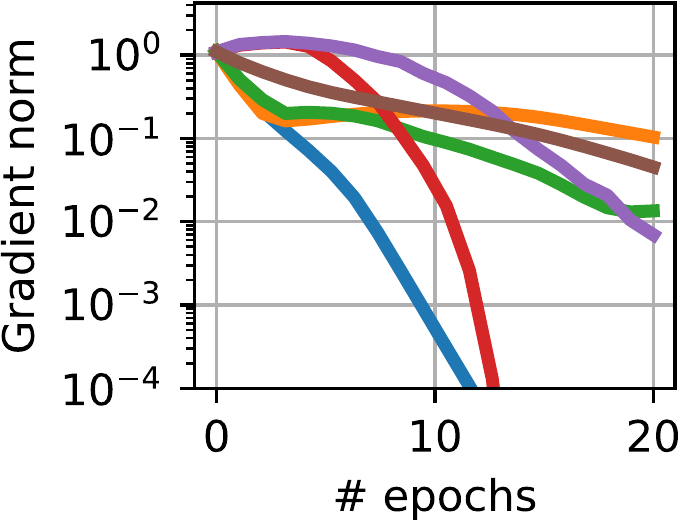}
  \end{minipage}
  \makebox[\columnwidth][c]{
  \begin{minipage}{0.65\linewidth}
    \begin{minipage}{\linewidth}
      \centering
      \includegraphics[width=0.7\linewidth]{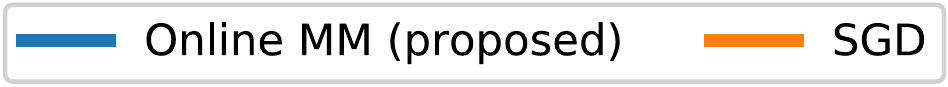}
    \end{minipage}
    \begin{minipage}{\linewidth}
      \centering
      \includegraphics[width=0.49\linewidth]{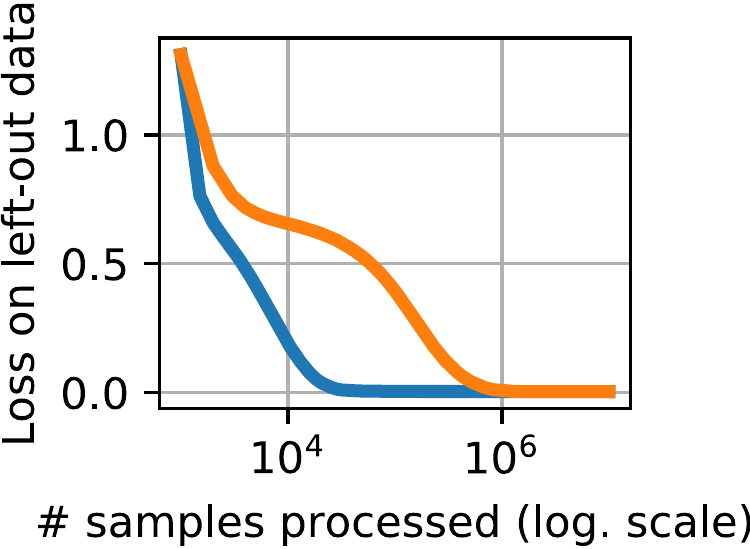}
      \includegraphics[width=0.49\linewidth]{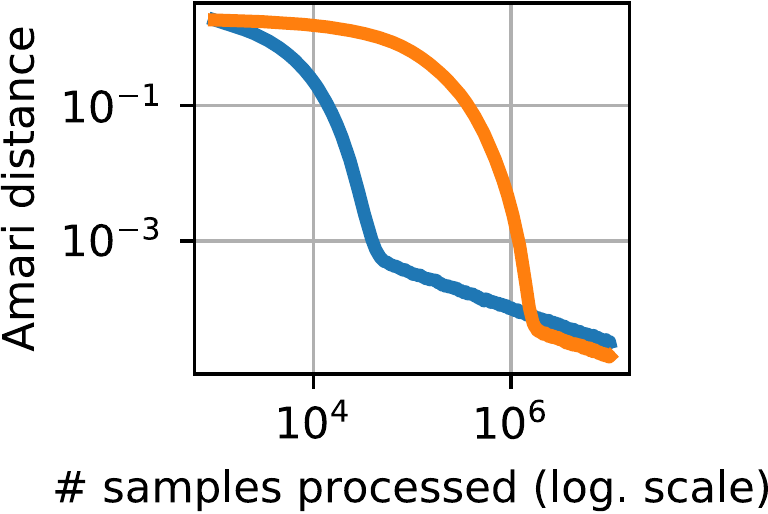}
    \end{minipage}
  \end{minipage}
  }
  \hfill

  \caption{Results on synthetic data. Top: finite-sum problem. $100$ datasets of size $	n = 10^6$ and $p= 10$ are generated, each algorithm performs 20 epochs (passes on the dataset).
     Bottom: online problem. $100$ datasets of size $n = 10^7$ and $p= 10$ are generated, each algorithm performs one pass on each dataset. Metrics are displayed with respect to epochs/number of passes.}
\label{fig:res_synth}
\end{figure}

\textbf{Real datasets:}
The algorithms are applied on classical ICA datasets, covering a wide range of dimensions $p$.
The first experiment is in the spirit of~\citep{hoyer2000independent}.

We extract a big 32GB dataset of $n=4 \times 10^7$ square patches of size $10 \times 10$ from natural images.
Each patch is vectorized into an array of dimension $p=100$.
Only the online algorithms are used to process this dataset since it does not fit into RAM.
The results on this dataset are displayed in Fig.~\ref{fig:res_huge}.
\begin{figure}
\includegraphics[width=0.6\linewidth]{online_synth_legend.pdf}
\centering
\includegraphics[width=0.8\columnwidth]{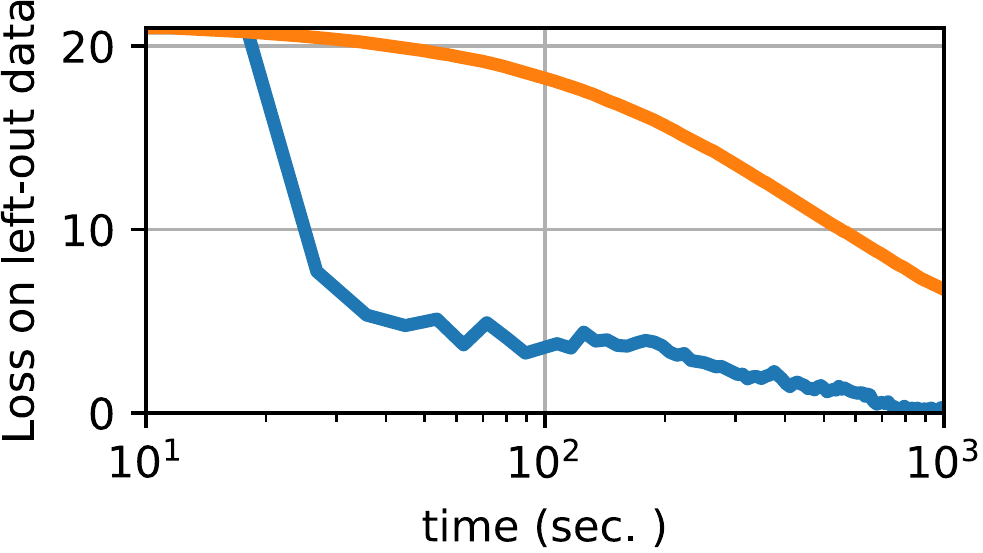}

\caption{Online algorithms applied on a 32 GB real dataset with $p=100$ and $n=4\times 10^7$. Time is in logarithmic scale. Values of the loss on left out data greater than its initial value are truncated.}
\label{fig:res_huge}
\end{figure}

We also generate smaller datasets in the same fashion, of size $n=10^6$, and $10 \times 10$ patches.
The dimension is reduced to $p=10$ using PCA.

Finally, an openly available EEG dataset~\citep{delorme2012independent} of dimension $p=71$, $n=10^6$ is used without dimension reduction.
Each signal matrix is multiplied by a $p \times p$ random matrix.
The different algorithms are applied on these datasets with 10 different random initializations, and for 50 epochs in the finite sum setting.
Results are displayed in Fig.~\ref{fig:res_real}.

\begin{figure}
  \begin{minipage}{\linewidth}
      \centering
      \includegraphics[width=0.63\linewidth]{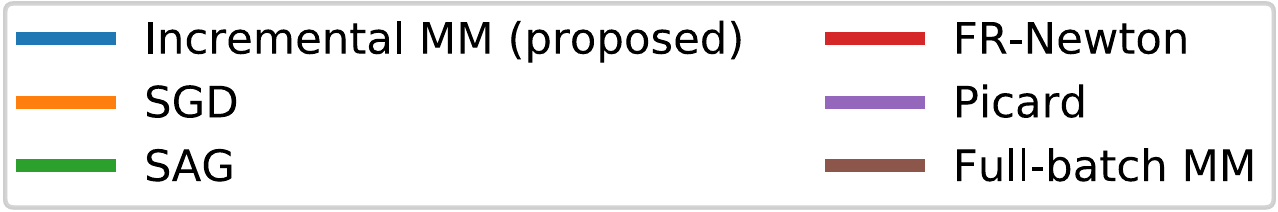}
      \includegraphics[width=0.35\linewidth]{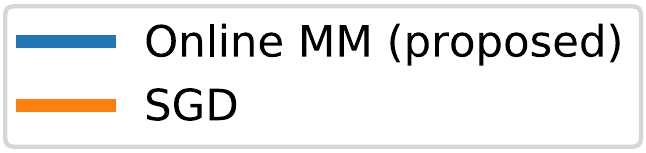}
  \end{minipage}
  \begin{minipage}{\linewidth}
      \centering      \includegraphics[width=0.32\columnwidth]{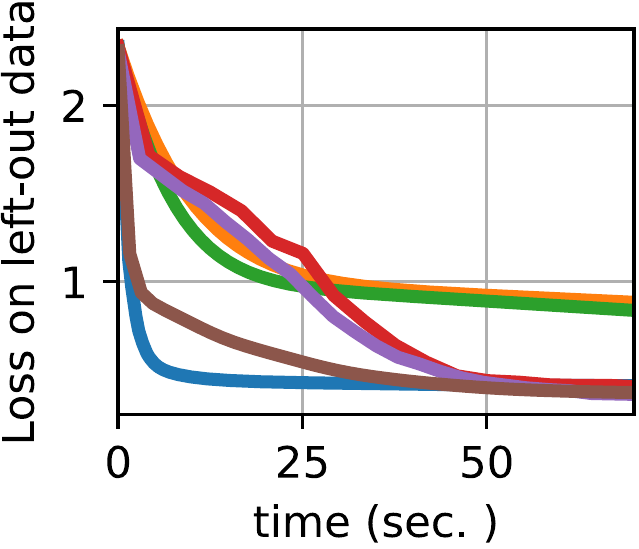}
      \includegraphics[width=0.32\columnwidth]{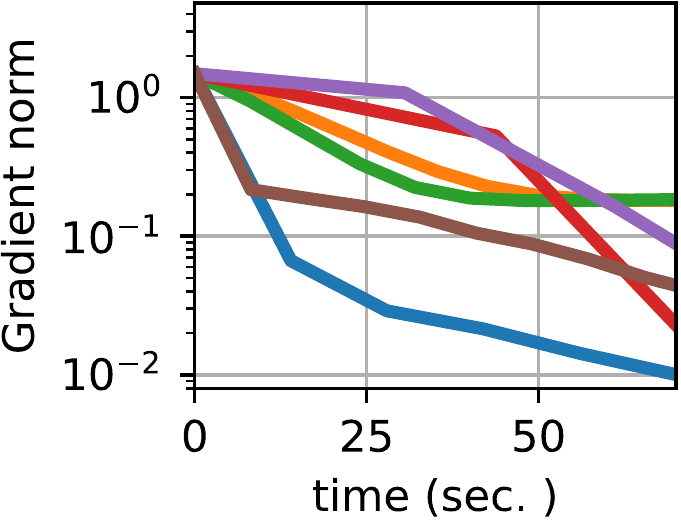}
      \includegraphics[width=0.32\columnwidth]{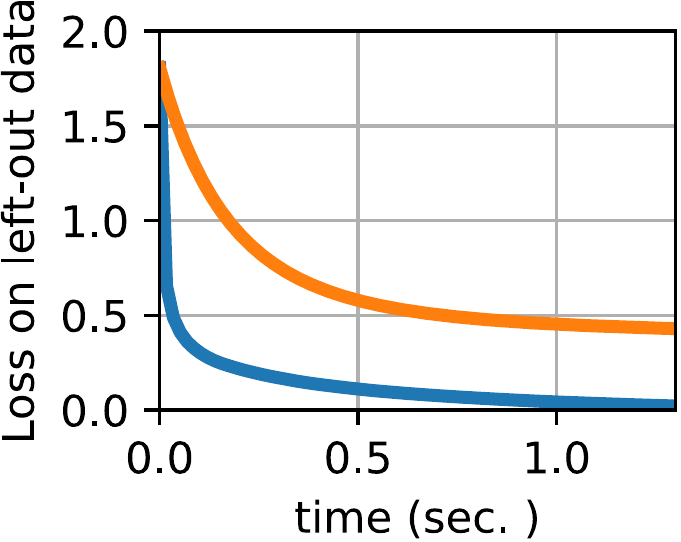}
  \end{minipage}
  \begin{minipage}{\linewidth}
      \centering
      \includegraphics[width=0.32\columnwidth]{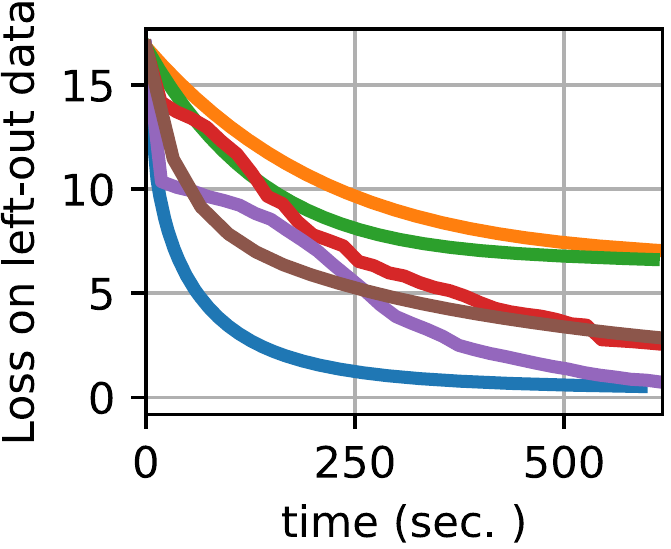}
      \includegraphics[width=0.32\columnwidth]{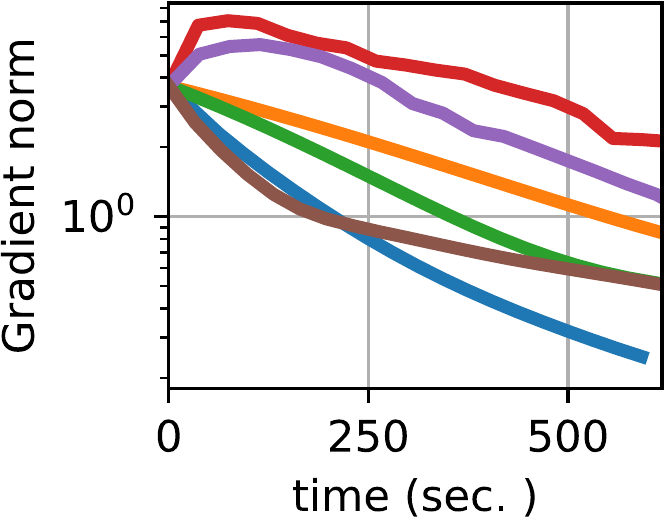}
      \includegraphics[width=0.32\columnwidth]{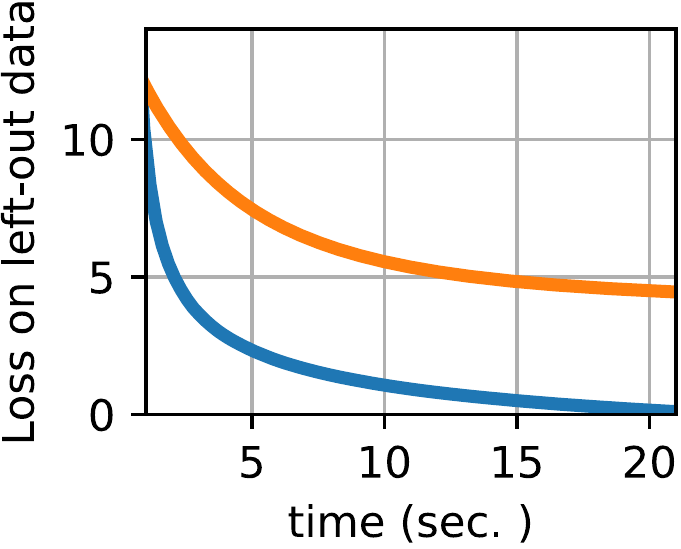}
  \end{minipage}

\caption{Behavior of different algorithms on real data. Top: 15 image patch datasets of size $10 \times 10^6$ are generated, and the averaged results are displayed. Bottom: same with 15 EEG datasets of size $71 \times 10^6$. Left and middle: finite sum problem. Right: online problem. Metrics are displayed with respect to time.}
\label{fig:res_real}
\end{figure}

\subsection{Discussion}

Experiments run on both synthetic and real data of various dimensions demonstrate that the proposed methods consistently perform best when quantifying the loss on left-out data (test error).
This metric is arguably the most important from a statistical machine learning standpoint.
This is also validated by the Amari distance in the simulated case: the proposed method shows similar convergence as FR-Newton and FastICA, and outperforms other algorithms.
Regarding the gradient norm metric (similar to training error), in the simulated and image patch experiment, the proposed algorithm is in the end slower than FR-Newton.
This behavior is expected: the incremental algorithm has a linear convergence, while second order methods are quadratic algorithm.

However, FR-Newton catches up with the proposed algorithm well after the testing error plateaus, so when the error of the model is dominated by the \emph{estimation error}~\citep{bottou2008tradeoffs} rather than the optimization error.

\textbf{Effect of the greedy update rule:}
On the $30 \times 10^5$ dataset (Fig.~\ref{fig:greed}), we run the incremental algorithm with the greedy coordinate update rule discussed in Sec.~\ref{sec:dual} with $q=1$ and $q=3$.
We compare it to a random approach (where $q$ random sources are updated at each iteration) for the same values of $q$, and to the more costly full-selection algorithm, where each source is updated for each sample.
The greedy approach only adds a negligible computational overhead linear in $p$ compared to the random approach, while leading to much faster estimation.
In terms of generalization error, it is only slightly outperformed by the full selection approach ($q=p$).

\begin{figure}
  \begin{minipage}{\linewidth}
      \centering
      \includegraphics[width=\linewidth]{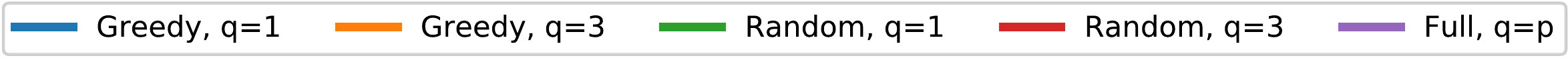}
  \end{minipage}
  \begin{minipage}{\linewidth}
      \centering
      \includegraphics[width=0.32\columnwidth]{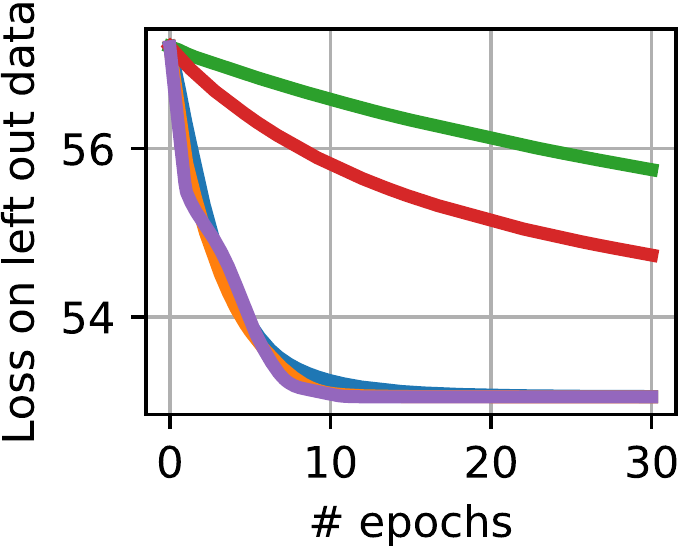}
      \includegraphics[width=0.32\columnwidth]{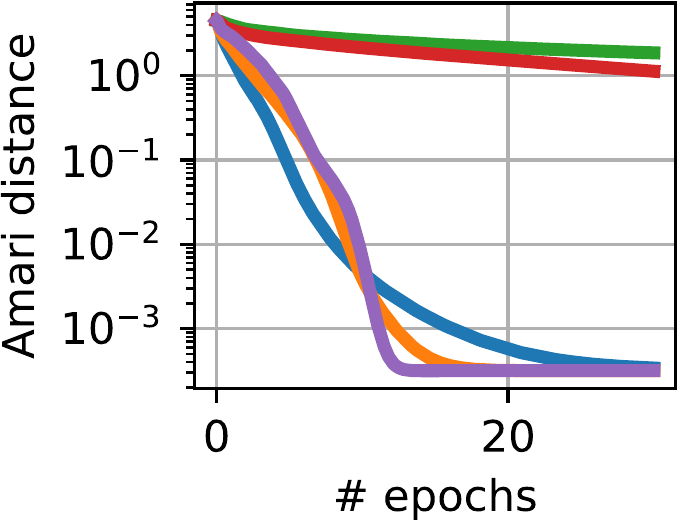}
      \includegraphics[width=0.32\columnwidth]{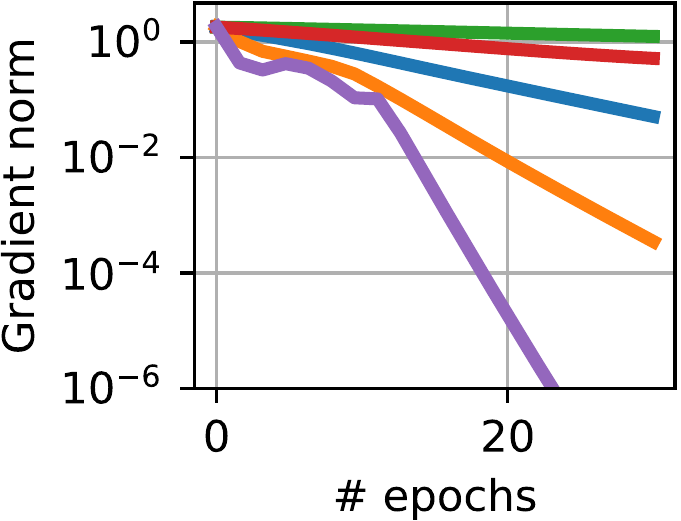}
  \end{minipage}

\caption{Effect of the greedy update rule, on a synthetic problem of size $p=30$, $n=10^5$. For a similar complexity, the greedy approach gives much faster convergence than the random approach.}
\label{fig:greed}
\end{figure}

\section{Conclusion}
In this article, we have introduced a new majorization-minimization framework for ICA, and have shown that it is equivalent to an EM approach for Gaussian scale mixtures.
Our method has the valuable advantage of guaranteeing a decrease of the surrogate loss function, which
enables stochastic methods with descent guarantees. This is, to the best of our knowledge, a unique feature for a stochastic ICA algorithm.
We have proposed both an incremental and an online algorithm for the finite-sum and online problems, with the same complexity as SGD thanks to an efficient greedy coordinate descent update.
Experiments show progress on current state-of-the-art, without the need for tedious manual setting of any parameter.
\medskip

\subsubsection*{Acknowledgments}
We acknowledge support from the European Research Council (grants SEQUOIA 724063 and SLAB 676943).

\bibliographystyle{apa}
\bibliography{biblio}

\newpage
\onecolumn

\appendix

This is the supplementary material for the AISTATS 2019 submission: ``Stochastic algorithms with descent guarantees for ICA''.

\section{Proof of equivalence of EM}
Given the Gaussian scale mixture formulation of $d$, as $d(y) = \int_0^{+\infty}
g(y, \eta)q(\eta) d\eta$, an EM algorithm would do the following:
\begin{eqnarray*}
\mathcal{L}_n(W)
& = &  - \logdet{W} - \frac1n\sum_{i=1}^p \sum_{j=1}^{n}[\log(d([W X]_{ij})) \\
& = &  - \logdet{W} - \frac1n \sum_{i=1}^p \sum_{j=1}^{n}\log(\int_0^\infty g([W X]_{ij}, \eta)q(\eta) d\eta)  \\
  & \leqslant &  - \logdet{W} - \frac1n\sum_{i=1}^p \sum_{j=1}^{n} \int_0^\infty r_{ij}(\eta)\log \frac{g([W X]_{ij}, \eta)q(\eta)}{r_{ij}(\eta)} d\eta \\
\end{eqnarray*}
where the $r_{ij}$ are any density functions.
The equality if and only if $r_{ij}(\eta)\propto g([W X]_{ij}, \eta)q(\eta)$.
The upper bound in the last equation can be written as:
$$
- \logdet{W} + \frac{1}{2n} \sum_{i=1}^p \sum_{j=1}^n \tilde{U}_{ij}\ [WX]_{ij}^2
\ + \ c \enspace,
$$
where $ \tilde{U}_{ij} = \int_0^\infty r_{ij}(\eta) \eta^{-1} d\eta$ and where $c$ is a remaining term which does not depend on $W$.
Thus, the upper bound has the same dependence on $W$ as the loss $\loss_n(W)$.
The E-step thus computes $\tilde{U}_{ij}= \int_0^\infty r_{ij}(\eta) \eta^{-1} d\eta$ using a density $r_{ij}\propto g([W X]_{ij}, \eta)q(\eta)$.
This exactly corresponds to the majorization step described in our algorithm.
The proof, from~\citep{palmer2006variational}, is as follows.
Dropping the indices for readability, and denoting $y=[W X]_{ij}$, normalizing $r$ gives:
$$
r(\eta) =  \frac{g(y, \eta)q(\eta)}{d(y)}
$$
so that
$$
\tilde{U}(y)= \frac{1}{d(y)}\int_0^\infty g(y, \eta)\eta^{-1} q(\eta) d\eta.
$$
Hence, using $\frac{\partial}{\partial y} g(y, \eta)= -g(y, \eta) y/ \eta$, we get
$$
\tilde{U}(y)= -\frac{1}{yd(y)}\int_0^\infty \frac{\partial}{\partial y} g(y, \eta)q(\eta) d\eta = -\frac{d'(y)}{yd(y)} = u^*(y) \enspace.
$$
This demonstrates that, although not formulated in this fashion, the proposed method is indeed equivalent to an EM algorithm.

\section{The EM algorithm for noisy mixtures is stuck in the noise-free limit}

We follow the update rules given in~\citep{bermond1999approximate,girolami2001variational}.
The model is $\bm{x} = A \bm{s} + \bm{n}$ where $\bm{n} \sim \mathcal{N}(0, \Sigma)$.
Key quantities for the update rule are the following expectations:
\begin{align}
\label{eq:exps}
\mathbb{E}[\bm{s} \lvert \bm{x}]
&= (A^{\top}\Sigma^{-1}A + \Lambda^{-1})^{-1}A^{\top} \Sigma^{-1}\bm{x}
\\
\mathbb{E}[\bm{s}\bm{s}^{\top}\lvert \bm{x}]
&= (A^{\top}\Sigma^{-1}A + \Lambda^{-1})^{-1} + \mathbb{E}[\bm{s} \lvert \bm{x}]\mathbb{E}[\bm{s} \lvert \bm{x}]^{\top} \enspace,
\label{eq:expss}
\end{align}
where $\Lambda$ is a diagonal matrix.
In the case considered in the present article, $A$ is square and invertible and $\Sigma = 0$.
Basic algebra shows that in that case, the above formula simplifies to:
\begin{equation}
\mathbb{E}[\bm{s} \lvert \bm{x}] = A^{-1}\bm{x} \enspace \text{and} \enspace
\mathbb{E}[\bm{s}\bm{s}^{\top}\lvert \bm{x}] = A^{-1}\bm{x} \bm{x}^{\top} A^{-\top}\enspace.
\label{eq:simples}
\end{equation}
The EM update for $A$ based on $n$ samples $\bm{x}_1, \cdots \bm{x}_n$:
then is
\begin{equation}
A^{\text{new}} = (\sum_{i=1}^{n} \bm{x}_i \mathbb{E}[\bm{s} \lvert \bm{x}_i])(\sum_{i=1}^{n} \mathbb{E}[\bm{s}\bm{s}^{\top}\lvert \bm{x}_i])^{-1} \enspace ,
\label{eq:noisya}
\end{equation}
which yields $A^{\text{new}} = A$ by using Eq.~\eqref{eq:simples}. The EM algorithm is thus frozen in the case of no noise.

\section{Proof of guaranteed descent}

Let us demonstrate that one iteration of the incremental algorithm~\ref{alg:inc} decreases $\loss_n$.
At the iteration $t$, let $W$ be the current unmixing matrix, $U^{\text{mem}}$ the state of the memory, and the $A^i$'s the current sufficient statistics.
As said in Section~\ref{sec:em}, E-step, we have $A^i_{kl} = \frac1n\sum_{j=1}^n U^{\text{mem}}_{ij} X_{kj}X_{kl}$.
Therefore, the algorithm is in the state $(W, U^{\text{mem}})$, and the corresponding loss is $\loss_n(W, U^{\text{mem}})$.
After the majorization step, the memory on the mini-batch is updated to minimize $\loss_n$.
Hence, the majorization step diminishes $\loss_n$.
Then, each descent move in the minimization step guarantees a decrease of $\loss_n$.
Both steps decrease $\loss_n$, the incremental algorithm overall decreases the surrogate loss function.

\section{Fast minimization step using conjugate gradient}

The minimization step~\eqref{eq:minw} involves computing the $i$-th row of the inverse of a given $ p \times p $ matrix $K$.
This amounts to finding $\bm{z}$ such that $K\bm{z} = \bm{e}_i$.
Exact solution can be found by Gauss-Jordan elimination with a complexity $O(p^3)$.
However, expanding the expression of $K$ yields $K_{kl} = \frac1n \sum_j U_{ij} Y_{kj} Y_{lj}$, where $Y = WX$.
\reb{If the rows of $Y$ are independent (as it is expected at convergence of the algorithm), then the off-diagonal elements of $K$ almost cancel.
Hence, the diagonal matrix $\textrm{diag}(K)$ is an excellent \emph{preconditioner} for solving the previous equation.
}{
It follows that:

\begin{lemma}
Assume that $W$ is such that the rows of $ Y=WX$ are independent. Then, $K_{kl} = O(\frac{1}{\sqrt{n}})$ for $k \neq l$.
\end{lemma}
\begin{proof}
In that case, $\mathbb{E}[\mathbf{u}_i\mathbf{y}_k\mathbf{y}_l] = 0$ since $\mathbf{y}_k$ and $ \mathbf{y}_l$ are independent.
Hence, the central limit theorem yields the advertised result.
\end{proof}

Therefore, the matrix $K$ is well approximated by its diagonal provided that $n$ is large enough, and that the current signals $Y$ are close enough from independence.
As such, we use the diagonal of $K$ as a \emph{preconditioner} to the conjugate gradient technique for solving $K\mathbf{z}=\mathbf{e}_i$.
}

This gives an excellent approximation of the solution in a fraction of the time taken to obtain the exact solution.

\end{document}